\documentclass[12pt]{article}

\usepackage{amsmath,amssymb,amsthm}
\usepackage{mathtools}
\usepackage{graphicx}
\usepackage{xcolor}
\usepackage{hyperref}
\hypersetup{colorlinks=true,linkcolor=blue,citecolor=blue,urlcolor=blue}
\usepackage[numbers,sort&compress]{natbib}
\usepackage{geometry}
\usepackage[expansion=false]{microtype}
\usepackage{tikz}
\usetikzlibrary{arrows.meta,positioning,shapes.geometric}
\usepackage{authblk}

\newtheorem{theorem}{Theorem}[section]
\newtheorem{proposition}[theorem]{Proposition}
\newtheorem{lemma}[theorem]{Lemma}

\theoremstyle{definition}
\newtheorem{definition}[theorem]{Definition}
\newtheorem{example}[theorem]{Example}
\theoremstyle{remark}
\newtheorem{remark}[theorem]{Remark}

\newcommand{\K}[2]{\mathcal{K}_{#1}(#2)}

\newcommand{\Prob}[2]{\mathsf{P}_{#1}(#2)}

\newcommand{\Bel}[3]{\mathsf{B}_{#1}(#2 \mid #3)}
\newcommand{\Guard}{\Gamma}
\newcommand{\permit}{\mathsf{permit}}
\newcommand{\deny}{\mathsf{deny}}
\newcommand{\LLMpol}{\pi_{\mathrm{LLM}}}

\newcommand{\Expertise}{\mathsf{E}}
\newcommand{\vocab}{\mathcal{L}}
\newcommand{\Lwn}{\mathcal{L}_{\mathsf{wn}}}
\newcommand{\vocabKsrc}{\vocab_{K,\mathrm{src}}}
\newcommand{\ELGM}{\mathcal{M}^{G}}
\newcommand{\Racc}[1]{R_{#1}}
\newcommand{\lot}[1]{L_{#1}}
\newcommand{\Agents}{\mathsf{Ag}}
\newcommand{\Acts}{\mathcal{A}}

\newcommand{\hist}{\mathbf{c}}

\begin{document}

\title{Epistemic-Probabilistic Model for Guarded Multi-Agent LLM Coordination}
\author[1]{Mehdi Nasiri}
\author[2]{Mohammad Saeed Arvenaghi}
\author[1]{Sadegh Vaezi}
\author[2,1]{Ebrahim Ardeshir-Larijani\footnote{Corresponding Author, larijani@iust.ac.ir}}

\affil[1]{Pasargad Institute for Advanced Innovative Solutions}
\affil[2]{Iran University of Science and Technology}
\date{}

\maketitle

\begin{abstract}
Large language model (LLM)-based multi-agent systems (MAS) are increasingly used in applied AI, yet their connection to established multi-agent theory remains underdeveloped. Many such systems lack explicit representations of social knowledge and protocol-governed coordination. La~Malfa et~al.~\cite{lamalfa2025} identify these issues among four shortcomings of current LLM-based MAS. We focus on the social-epistemic and coordination aspects of this gap.

We introduce Epistemic Probabilistic Language Agents (EPLA model), a neuro-symbolic architecture for multi-agent coordination under uncertainty. EPLA model combines an Epistemic Logic Core, a Conditional Belief Engine, retrieval-augmented memory, a Policy LLM, and a Symbolic Guard that controls execution against the authoritative symbolic state. Guard feedback, Adversarial Representation Engineering (ARE), and reinforcement learning with linear temporal logic (LTL) objectives provide interfaces for later adaptation. We formalize the epistemic layer in a gossip testbed through epistemic lottery gossip models, combining view-based call histories with agent-indexed probability weights. For a restricted knowledge fragment, we prove lottery transparency and invariance under positive admissible reweighting. For a precisely defined source-compatible gossip instance, an explicit modal-depth-one translation into Apt and Wojtczak's decidable language yields decidability of source-compatible Guards. We also prove a conditional ranking-progress bound for stochastic action selection.
\end{abstract}


\section{Introduction}
\label{sec:intro}

Large language model (LLM)-based multi-agent systems are often engineered as message-passing workflows.  An agent may recommend a tool call, delegate to another agent, or forward information retrieved from a document.  Fluent generation alone does not establish that the action is licensed by the agent's information state or by the interaction protocol.  In the position paper \emph{Large Language Models Miss the Multi-Agent Mark}, La~Malfa et~al.~\cite{lamalfa2025} identify a recurring mismatch between LLM-agent practice and multi-agent theory: current systems use multi-agent terminology, but often omit explicit social epistemics and protocol-governed coordination.

EPLA model addresses this gap through a neuro-symbolic division of responsibility: learned components propose and adapt, while explicit symbolic components represent state and control execution.  The policy module uses a large language model (Policy LLM) to interpret language and propose typed communication or tool actions.  The Epistemic Logic Core (ELC) maintains the authoritative symbolic state.  A Conditional Belief Engine (CBE) represents graded and conditional uncertainty, drawing on conditional-belief logic~\citep{vaneijck2017conditional}.  Retrieval-augmented generation (RAG) supplies retrieved context~\citep{lewis2020rag}; EPLA model additionally requires retained records to be durable and source-linked.  The Symbolic Guard determines whether a candidate action may execute.  Guard feedback is intended to supervise Adversarial Representation Engineering (ARE), an empirical model-editing method~\citep{zhang2024are}.  A reinforcement-learning (RL) head may optimize long-horizon behavior against objectives expressed in linear temporal logic (LTL)~\citep{hammond2021marl,camacho2019ltl}.  These are intended interfaces and empirical hypotheses, not implementation results.  For example, an LLM may propose a call between two agents, but the Guard permits it only when the exact symbolic state satisfies the call's specified precondition.

The formal setting is gossip-structured coordination: agents communicate pairwise, calls change what agents know, and later actions are permitted only when their epistemic preconditions are satisfied.  This setting is deliberately controlled, not a claim that arbitrary tool traces already have gossip semantics.  Epistemic gossip protocols have precise call-history semantics~\citep{ditmarsch2017gossip,ditmarsch2019dynamic}, while epistemic probability logic simplified (EPLS) gives a compact semantics in which knowledge is probability one~\citep{vaneijck2014epls}.  Combining these ideas distinguishes an LLM's graded uncertainty from a Guard's crisp knowledge condition.  The narrow formal question is whether adding positive uncertainty weights to source-style gossip histories changes any crisp knowledge-based Guard decision.

This paper makes two contributions.  The contribution is limited to the EPLA-specific architecture specification, formalization, and stated derivations; the paper does not claim priority for the individual methods or proof techniques in isolation.

\begin{enumerate}
\item It specifies an intended EPLA model architecture in which an exact symbolic state, a conditional-belief state, and source-linked retrieval records inform a Policy LLM; a Symbolic Guard controls execution and returns diagnostic feedback; and representation editing and temporal-objective policy learning may use that feedback to shape later proposals.
\item It defines \emph{epistemic lottery gossip models} (ELGMs), which enrich gossip call-history models with agent-indexed lottery weights.  Their accessibility relation is based on an agent's view, not merely on the subsequence of calls involving that agent.  The paper proves lottery transparency and positive-reweighting invariance, transfers Guard decidability only for an explicit modal-depth-one source fragment, and gives a conditional ranking bound for stochastic action selection.
\end{enumerate}

Because this version is an extended abstract, the results are accompanied by proof sketches; the full manuscript contains the complete proofs.

The formal and architectural components have different roles.  Theorems about ELGMs concern the stated gossip abstraction; they do not imply that an implementation realizes that abstraction.  The architecture locates the intended feedback paths, while the effects of its learned components remain empirical questions.

\section{EPLA model Architecture}
\label{sec:architecture}

Figure~\ref{fig:epla} summarizes the EPLA model architecture.
EPLA model treats coordination as a controlled loop rather than a free-form exchange of messages.  At time $t$, an agent receives an observation, consults persistent memory, forms a candidate action, and submits that action to the Guard.  The Guard checks the authoritative state, not a summary generated by the LLM.  An accepted action is passed to the ELC for the state transition; a rejected action returns a diagnostic and structured labels that may shape later proposals.

The responsibility boundaries are deliberately asymmetric.  The ELC alone maintains and commits the authoritative epistemic state; the CBE maintains graded uncertainty; and retrieval returns bounded, source-linked evidence. The Policy LLM turns these inputs into typed candidate actions, but cannot execute them.  The Guard checks the exact state and returns \textsc{permit} or \textsc{deny}; its feedback supplies diagnostic labels that ARE and the RL/LTL module may use to shape later proposals, without overriding the Guard.

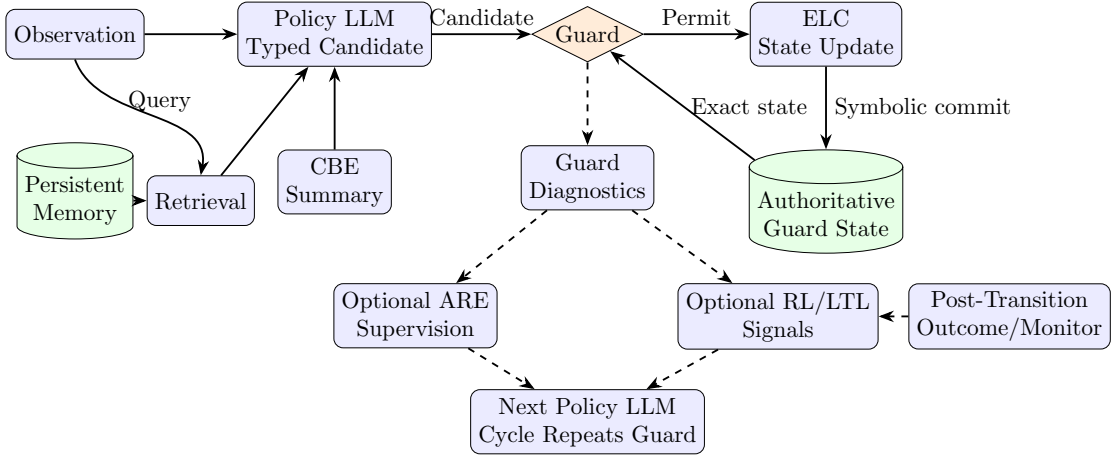
\begin{figure}[ht]
\centering
\footnotesize
\resizebox{0.98\columnwidth}{!}{%
\begin{tikzpicture}[
>=Stealth,
node distance=1.0cm and 1.1cm,
font=\footnotesize,
block/.style={
draw,
rounded corners,
fill=blue!8,
minimum width=2.0cm,
minimum height=0.75cm,
align=center
},
storage/.style={
draw,
cylinder,
shape border rotate=90,
aspect=0.22,
minimum width=1.2cm,
minimum height=0.75cm,
fill=green!10,
align=center
},
decision/.style={
diamond,
draw,
fill=orange!15,
aspect=2,
inner sep=1pt,
align=center
},
line/.style={->,thick},
feedback/.style={->,thick,dashed}
]

\node[block] (obs) {Observation};

\node[block,right=1.4cm of obs] (llm)
{Policy LLM\\Typed Candidate};

\node[decision,right=1.5cm of llm] (guard)
{Guard};

\node[block,right=1.6cm of guard] (elc)
{ELC\\State Update};

\node[storage,below=1.25cm of obs] (mem)
{Persistent\\Memory};

\node[block,right=0.2cm of mem,minimum width=1.4cm] (ret)
{Retrieval};

\node[block,below=1.25cm of llm,minimum width=1.7cm] (cbe)
{CBE\\Summary};

\node[storage,below=1.25cm of elc] (state)
{Authoritative\\Guard State};

\node[block,below=1.25cm of guard] (diag)
{Guard\\Diagnostics};

\node[block,below left=1.1cm and 0.35cm of diag] (are)
{Optional ARE\\Supervision};

\node[block,below right=1.1cm and 0.35cm of diag] (rl)
{Optional RL/LTL\\Signals};

\node[block,right=0.45cm of rl] (outcome)
{Post-Transition\\Outcome/Monitor};

\node[block,below=2.7cm of diag] (later)
{Next Policy LLM\\Cycle Repeats Guard};

\draw[line] (obs)--(llm);
\draw[line] (llm)--node[above]{Candidate}(guard);
\draw[line] (guard)--node[above]{Permit}(elc);

\draw[line] (mem)--(ret);
\draw[line] (ret)--(llm);
\draw[line] (cbe)--(llm);
\draw[line] (obs.south) to[out=-70,in=80] node[pos=0.55,above]{Query} (ret.north);

\draw[line] (state)--node[right]{Exact state}(guard);
\draw[line] (elc)--node[right]{Symbolic commit}(state);

\draw[feedback] (guard)--(diag);
\draw[feedback] (diag)--(are);
\draw[feedback] (diag)--(rl);
\draw[feedback] (outcome)--(rl);
\draw[feedback] (are)--(later);
\draw[feedback] (rl)--(later);

\end{tikzpicture}
}
\caption{Simplified single-step architecture of EPLA model. The observation supplies both the Policy LLM context and the retrieval query. The Policy LLM proposes a typed candidate using the observation, retrieved memory, and a CBE summary. The Guard checks the candidate against the authoritative Guard state, which contains the exact symbolic and protocol/resource predicates required for execution; the CBE summary remains non-authoritative context. The ELC commits an accepted symbolic update. Guard diagnostics may supervise ARE and inform RL/LTL, while post-transition outcomes and the next LTL monitor state provide an additional RL/LTL input. Every later candidate re-enters the ordinary Guard path. Event-driven memory and CBE updates and external effects are specified in the execution model.}
\label{fig:epla}
\end{figure}


\subsection{Execution Model}

Let the operational state be
\[
  z_t=(G_t,\Sigma_t,\Pi_t,\mathcal{R}_t,q_t),
\]
where $G_t$ is the ELC state, $\Sigma_t$ is the CBE state, $\Pi_t$ is the protocol and resource state, $\mathcal{R}_t$ is the retrieval memory, and $q_t$ is the state of an LTL monitor when a temporal objective is active.  One decision cycle has nine stages:
\begin{enumerate}
  \item observe the current task and multi-agent state;
  \item retrieve relevant events, messages, rules, and evidence from $\mathcal{R}_t$;
  \item construct a bounded context that identifies the current task, symbolic summary, and retrieved evidence;
  \item ask the Policy LLM for a typed candidate action;
  \item check the action against epistemic, protocol, consistency, and budget conditions;
  \item return a denial and diagnostic, or prepare an accepted transition without yet performing an external side effect;
  \item atomically record the accepted action and let the ELC commit $G_{t+1}$; then perform any external side effect under a durable intent record and update $\Sigma_t$, $\Pi_t$, and $\mathcal{R}_t$ from the resulting event;
  \item when ARE is enabled, form representation-level supervision from the Guard labels; and
  \item when the RL head is enabled, update its policy and value estimates from the next monitor state and acceptance signal.
\end{enumerate}

A conforming implementation must evaluate the Guard and commit the accepted ELC successor against the same versioned state.  It must either serialize those operations or revalidate the action when the state version changes.  External side effects also require a durable intent record and an idempotent or compensating execution rule; these concurrency and recovery mechanisms are requirements, not features established by the present formal model.

The formal model has no human override.  At the architecture level, a denied request for a policy exception, or a high-impact action whose semantics are not formalized, is sent to a designated human reviewer.  Human approval may authorize a new typed action or a rule change, but it does not make a failed Guard predicate true; any approved action must enter through the Guard and ELC again.

The action schema separates natural-language generation from state transition.  A candidate action includes an action type, participating agents, structured arguments, an evidence reference set, and an optional natural-language realization.  The ELC determines the operational effect of an accepted action.  The architecture retains three communication action types: \textsc{Call}, \textsc{Message}, and \textsc{Announce}.  This paper gives operational semantics only to a pairwise push-pull \textsc{Call}, because that action has both a precise secret-union transition and source-verified view semantics~\citep{ditmarsch2017gossip,apt2017epistemic}.  Formalizing \textsc{Message} or \textsc{Announce} would additionally require a content language, recipient and observability rules, a truth and ambiguity policy, and exact preconditions and state-update rules.  Those choices are not fixed in this draft, so the two action types remain in the architecture rather than receiving invented formal semantics.

\subsection{Conditional Belief, Retrieval, and Guard Feedback}

The CBE maintains qualitative conditional-belief queries of the form $\Bel{a}{\varphi}{\psi}$, following van Eijck and Li~\citep{vaneijck2017conditional}, and may separately maintain numerical belief estimates as proposed for EPLA model.  The cited conditional-belief operator itself is not a numerical degree.  CBE state is updated after communication and may be summarized for the Policy LLM.  The exact symbolic state, rather than a lossy summary, remains the reference used by the Guard.  The ELGM developed below supplies a precise formal abstraction of graded uncertainty over gossip histories; it does not assume that every CBE implementation uses the same representation.  In that formalization, theorem-backed Guard predicates are restricted to a crisp ELC knowledge fragment defined in Section~\ref{sec:prelim}; numerical or conditional CBE queries require separate semantics and proof obligations.

RAG supplies persistent context that would otherwise be lost from a bounded prompt.  Each memory record should retain an event identifier, source or provenance link, timestamp or logical order, access scope, and confidence or validation metadata.  Retrieval is useful for reconstructing task context and protocol rules, but retrieved text is evidence for a decision, not a substitute for the state predicates checked by the Guard.

The Guard has two outputs.  Its execution output is a binary decision over the typed action.  Its learning output is a critique that names the failed condition, the relevant state facts, and a machine-readable label such as \textsc{call-permission}, \textsc{truthfulness}, or \textsc{protocol-compliance}.  This second output permits failures to become diagnostic data rather than silent discarded candidates.

\subsection{Representation Editing and Temporal Policy Learning}

ARE and LTL-constrained RL address different parts of the architecture.  EPLA model proposes using Guard-derived target and anti-target labels to edit representations associated with behavioral concepts such as permission, truthfulness, and protocol compliance.  Whether those labels produce valid and effective ARE supervision is an empirical hypothesis.  ARE does not replace the Guard, and it is not identified with the lottery-reweighting operator in Section~\ref{sec:updates-llm}.

The RL head addresses action selection over a trajectory.  For example, a temporal objective may require protocol violations never to occur and every agent eventually to learn every secret.  The monitor state and Guard outcome then provide signals for optimizing policy behavior over multiple decisions.  ARE and RL/LTL occupy different architectural roles, and neither is assumed to replace the Guard.

\subsection{Scalability and Conformance Requirements}

We identify four engineering issues that a future implementation should address.  First, Guard checks may need tiered validation, caching, incremental state updates, or compiled monitors so that exact checking does not become an uncontrolled bottleneck.  Second, the ELC/CBE state and retrieval memory may need event-sourced logs, replayable updates, and checks for disagreement between compact Policy summaries and exact Guard state.  Third, multi-agent scaling requires explicit topology and memory-partition choices rather than an assumption of all-to-all communication.  Fourth, ARE and RL updates need an evaluated training schedule that can detect and limit interference; possible strategies include timescale separation, partial freezing, and policy anchoring.  These are proposed implementation strategies and evaluation obligations, not methods or performance properties established by the formal results.

For the formal analysis below, we project the architecture to a call-only abstraction: retain the accepted pairwise \textsc{Call} history, secret facts, agent views, and lottery weights, and omit budgets, retrieved text, free-form messages and announcements, ARE parameters, and the LTL monitor.  At the secret-only root, a permitted \textsc{Call}$(a,b)$ appends $ab$ and the ELC replaces both callers' secret sets by their union; this is the transition modeled below.  \textsc{Message} and \textsc{Announce} candidates have no transition in the current formal model.  This projection explains the theorem boundary; it is not evidence that an implementation conforms to it.

For example, start with three agents, each of whom initially knows only its own secret.  The Policy LLM proposes a call between agents $a$ and $b$, with the precondition that $a$ does not yet know $b$'s secret.  The Guard reads the root ELC state, finds the call structurally available and the precondition true, and permits it.  The ELC records the call and updates both callers so that each knows both secrets.  If the same precondition is used for a second call between $a$ and $b$, the Guard denies it because $a$ now knows $b$'s secret; the history is unchanged and the failed-precondition label is returned as diagnostic feedback.

\section{Preliminaries}
\label{sec:prelim}

Let $\Agents$ be a finite set of agents.  Each agent initially owns one secret, and the call alphabet is $C:=\{ab:a,b\in\Agents,\ a\neq b\}$, with $ab$ denoting an unordered pair.  A finite call history is a word $\hist\in C^{<\omega}$.  The factual state used here records, for each agent, which secrets are locally available; phone-number exchange is outside the present formal instance.  In the push-pull setting used in the standard gossip literature, a call $ab$ makes the callers exchange the secrets they currently know.

The relevant epistemic object is the agent's \emph{view}.  The view is richer than the bare subsequence of calls in which the agent participates: it also records the local information the agent obtains through those calls.  This is the notion used in source gossip semantics~\citep{ditmarsch2017gossip,apt2017epistemic}.

\begin{definition}[Views and source-style accessibility]
\label{def:views}
For each agent $a$, let $\mathsf{view}_a(\hist)$ be the sequence of local observations available to $a$ after history $\hist$: calls not involving $a$ are invisible, while a call involving $a$ appends the call together with $a$'s resulting local gossip state.  Define
\[
  \hist \Racc{a} \hist' \quad\text{iff}\quad
  \mathsf{view}_a(\hist)=\mathsf{view}_a(\hist').
\]
\end{definition}

This definition intentionally follows the view-based semantics of epistemic gossip.  Equality of the call subsequence involving $a$ is not enough: if $a$ calls $c$ after $c$ has already learned another secret, then $a$ observes a different local state than in a history where $c$ had not learned it.

We use atoms $S_{ab}$ to mean that agent $a$ is familiar with agent $b$'s secret at the current history.  Let $V(S_{ab})$ be the set of histories where this is true.  These atoms are local for their first index: by construction of the view, if $\hist\Racc{a}\hist'$, then $S_{ab}$ has the same truth value at $\hist$ and $\hist'$.  The factual expertise condition is
\[
  \Expertise_0 \;:=\; \bigwedge_{a,b\in\Agents} S_{ab}.
\]
EPLS represents epistemic probabilities by lotteries over possible worlds~\citep{vaneijck2014epls}.  We use the same idea, but with histories as worlds.

\begin{definition}[Admissible epistemic lottery gossip model]
\label{def:admissible-elgm}
An admissible epistemic lottery gossip model is a tuple
\[
  \ELGM=(H,V,\{\Racc{a}\}_{a\in\Agents},\{\lot{a}\}_{a\in\Agents})
\]
where:
\begin{enumerate}
\item $H\subseteq C^{<\omega}$ is a nonempty set of finite histories;
\item $V$ is the factual valuation induced by the gossip state after each history;
\item $\Racc{a}$ is the view-equivalence relation from Definition~\ref{def:views}, restricted to $H$;
\item $\lot{a}:H\to\mathbb{R}_{>0}$ is a strictly positive lottery such that, for every $\hist\in H$,
\[
  0 < \sum_{\hist'\in [\hist]_a} \lot{a}(\hist') < \infty,
  \qquad [\hist]_a:=\{\hist'\in H: \hist\Racc{a}\hist'\}.
\]
\end{enumerate}
\end{definition}

\begin{definition}[Knowledge fragment]
\label{def:vocab-k}
The guard fragment $\vocab_K$ is generated by
\[
  \varphi ::= \top \mid p \mid \neg\varphi \mid (\varphi\wedge\psi) \mid \K{a}{\varphi},
\]
where $p$ ranges over factual gossip atoms.  It contains no numerical probability atoms such as $\Prob{a}{\varphi}\ge q$.  The source-compatible subfragment below uses only atoms $S_{ab}$.
\end{definition}

For atoms $p$, write $\ELGM,\hist\models p$ exactly when $\hist\in V(p)$.  The clauses for $\top$, negation, and conjunction are the usual Boolean clauses.  These clauses, together with the knowledge clause below, define satisfaction for the knowledge fragment.

The finite-sum condition is not cosmetic.  If $H$ contains an infinite $a$-accessibility class, a constant-weight prior on that class is not admissible; a finite protocol domain or a summable length-decaying prior is therefore needed.  EPLA model's use of positive real weights is a modeling generalization.  EPLS itself permits finite or countable worlds and defines lotteries with positive rational values bounded on every $R_a$-equivalence class; countability does not require real weights.  Finite rational instances are EPLS-compatible special cases.

For an admissible ELGM, define the probability assigned by $a$ at $\hist$ to a formula $\varphi$ by
\[
  \Prob{a}{\varphi}(\hist)
  =
  \frac{\sum_{\hist'\in [\hist]_a,\; \ELGM,\hist'\models\varphi}\lot{a}(\hist')}
       {\sum_{\hist'\in [\hist]_a}\lot{a}(\hist')}.
\]
Knowledge is probability one:
\[
  \ELGM,\hist\models \K{a}{\varphi}
  \quad\text{iff}\quad
  \Prob{a}{\varphi}(\hist)=1.
\]

Because the atoms $S_{ab}$ are local for their first index, this knowledge semantics gives $S_{ab}\leftrightarrow\K{a}{S_{ab}}$.  Thus $\Expertise_0$ is equivalent to the usual knowledge formulation $\bigwedge_{a,b\in\Agents}\K{a}{S_{ab}}$ in this model.

For the source-compatible decidability theorem only, we use the narrower gossip language $\Lwn$ defined by Krzysztof R. Apt and Dominik Wojtczak in \emph{On Decidability of a Logic of Gossips} and \emph{Common Knowledge in a Logic of Gossips}~\citep{apt2016decidability,apt2017epistemic}.  In logic, a language is a grammar: it specifies which formulas are permitted.  $\Lwn$ permits factual gossip formulas and an epistemic or common-knowledge modality applied directly to a factual formula, but it does not permit probability formulas or a modality inside another modality.  Thus the source-compatible formula $\K{a}{S_{bc}}$ is permitted, whereas $\K{a}{\K{b}{S_{cd}}}$ is not.  The general EPLA model Guard fragment $\vocab_K$ is different: it permits nested individual-knowledge formulas but excludes numerical probability atoms.  Definition~\ref{def:source-guard-fragment} defines $\vocabKsrc$, the smaller part of $\vocab_K$ that is translated into $\Lwn$.

\begin{definition}[Apt--Wojtczak source instance]
\label{def:aw-source-instance}
Fix a finite, linearly ordered agent set $A$ with $|A|\geq 3$ and a bijection $b\mapsto p_b$ from agents to distinct secrets.  Let $C_{\mathrm{src}}$ contain exactly the canonically written pairs $ab$ with $a<b$, and let $H_{\mathrm{src}}:=C_{\mathrm{src}}^{<\omega}$ be the full set of finite call sequences.  The initial factual state is the secret-only root in which each agent $a$ knows exactly $p_a$, and every call applies the source push-pull union update.

An \emph{Apt--Wojtczak source instance} is an admissible ELGM whose history domain is $H_{\mathrm{src}}$, whose valuation satisfies
\[
  \hist\in V(S_{ab})\quad\text{iff}\quad p_b\in \hist(\mathsf{root})_a,
\]
where $\hist(\mathsf{root})_a$ denotes agent $a$'s secret set after executing $\hist$ from the initial root state, and whose accessibility relation is exactly the source relation $\sim^{\mathrm{AW}}_a$ on call sequences.  Equivalently, it is equality of Apt--Wojtczak's recursively constructed source views, which their equivalence theorem identifies with $\sim^{\mathrm{AW}}_a$~\citep{apt2017epistemic}.  This specialized instance fixes Definition~\ref{def:views} to the source's secret-only call and view semantics; it does not include phone-number exchange, protocol-restricted history domains, or other EPLA model action types.
\end{definition}

\begin{definition}[Source-compatible Guard fragment and translation]
\label{def:source-guard-fragment}
Let the factual source formulas $\theta$ and the source-compatible Guard formulas $\varphi$ be generated by
\[
\begin{aligned}
\theta &::= \top \mid S_{ab} \mid \neg\theta \mid (\theta\wedge\theta),\\
\varphi &::= \theta \mid \neg\varphi \mid (\varphi\wedge\varphi) \mid \K{a}{\theta}.
\end{aligned}
\]
Write $\vocabKsrc$ for this modal-depth-one subfragment of $\vocab_K$.  For a fixed $a_0\in A$, define a translation $\tau$ into the Apt--Wojtczak language recursively by
\[
\begin{aligned}
\tau(S_{ab})&=F_a p_b, &
\tau(\top)&=F_{a_0}p_{a_0}\lor\neg F_{a_0}p_{a_0},\\
\tau(\neg\varphi)&=\neg\tau(\varphi), &
\tau(\varphi\wedge\psi)&=\tau(\varphi)\wedge\tau(\psi),\\
\tau(\K{a}{\theta})&=K_a\tau(\theta).
\end{aligned}
\]
Here $F_a p_b$ is the source familiarity atom stating that agent $a$ is familiar with agent $b$'s secret, and $K_a$ is the singleton-agent epistemic modality (equivalently, $C_{\{a\}}$ in the source notation).  The symbol $\lor$ is the usual Boolean abbreviation.  Factual formulas translate into the source propositional language, and each knowledge formula translates to one unnested individual-knowledge modality over a propositional formula.  Hence the image of $\vocabKsrc$ lies in $\Lwn$.  The translation introduces neither probability terms nor nonsingleton common-knowledge operators.
\end{definition}

\section{Basic Epistemic Properties}
\label{sec:elgm}

The ELGM adds graded uncertainty to a standard gossip model without changing the underlying view-based accessibility relation.  The lottery weights say which histories an agent takes to be more likely among the histories compatible with its view.  The knowledge operator remains crisp because knowledge is the probability-one case.

\begin{proposition}[S5 behavior of knowledge]
\label{prop:s5-elgm}
In every admissible ELGM, the operator $\K{a}{\cdot}$ satisfies the S5 axioms over the language in Definition~\ref{def:vocab-k}.
\end{proposition}

\begin{proof}[Proof sketch]
Since $\Racc{a}$ is an equivalence relation, every world in $[\hist]_a$ has the same accessible class.  Strictly positive weights make probability one equivalent to truth at every accessible history.  The usual S5 argument for equivalence relations therefore applies. 
\end{proof}

\begin{example}[A view distinction]
\label{ex:view-distinction}
Let the agents be $a,b,c$.  Compare histories $ac$ and $bc;ac$.  The subsequence of calls involving $a$ is $ac$ in both histories.  But after $ac$, agent $a$ learns only what $c$ had at that point.  In $bc;ac$, agent $c$ has first learned from $b$, so the later call gives $a$ more information.  Hence the views of $a$ differ.  This is why Definition~\ref{def:views} uses views rather than call subsequences.
\end{example}

\section{Guards and Lottery Transparency}
\label{sec:guards}

An action guard is an epistemic precondition.  In an LLM-agent interpretation, the LLM may propose an action, but the symbolic layer decides whether the action is permitted.  The formal core models pairwise gossip calls only.  For $\hist\in H$, let
\[
  \Acts_H(\hist):=\{ab\in C:\hist\cdot ab\in H\}.
\]
Thus a call is structurally executable only when its successor history is in the model domain.  Messages, retrieval operations, and other actions in the wider EPLA model architecture require a separate operational semantics.

\begin{definition}[Symbolic guard]
\label{def:guard}
Let each call $\alpha\in C$ have a precondition $\mathsf{pre}(\alpha)\in\vocab_K$.  The Guard in $\ELGM$ is
\[
  \Guard_{\ELGM}(\alpha,\hist)=
  \begin{cases}
    \permit, & \text{if } \alpha\in\Acts_H(\hist)\text{ and }\ELGM,\hist\models \mathsf{pre}(\alpha),\\
    \deny, & \text{otherwise.}
  \end{cases}
\]
\end{definition}

Guard soundness is immediate from the definition: if $\Guard_{\ELGM}(\alpha,\hist)=\permit$, then the stated precondition is true at $\hist$ and the call has a successor history in $H$.  This is a semantic property of the formal predicate.  A concrete Guard must conform to the same state representation, parser, and rule semantics for the property to govern its executions.  The nontrivial point is that, for $\vocab_K$, the truth value does not depend on the exact positive lottery weights.

Let $M_0=(H,V,\{\Racc{a}\}_{a\in\Agents})$ be the lottery-free Kripke model underlying $\ELGM$, using the standard Kripke clause for $K_a$.

\begin{lemma}[Lottery transparency]
\label{lem:lottery-transparency}
For every admissible ELGM, every $\hist\in H$, and every $\varphi\in\vocab_K$,
\[
  \ELGM,\hist\models\varphi
  \quad\text{iff}\quad
  M_0,\hist\models\varphi.
\]
\end{lemma}

\begin{proof}[Proof sketch]
The proof is by structural induction.  Atoms and Boolean cases are immediate because $\ELGM$ and $M_0$ share $H,V,$ and $\Racc{a}$.  For $\K{a}{\psi}$, admissibility gives strictly positive weights and a finite positive denominator on $[\hist]_a$.  Therefore $\Prob{a}{\psi}(\hist)=1$ iff no accessible history falsifies $\psi$.  By the induction hypothesis, this is equivalent to the Kripke truth condition for $K_a\psi$ in $M_0$. 
\end{proof}

\begin{lemma}[Source correspondence]
\label{lem:source-correspondence}
For every Apt--Wojtczak source instance, every $\hist\in H_{\mathrm{src}}$, and every $\varphi\in\vocabKsrc$,
\[
  \ELGM,\hist\models\varphi
  \quad\text{iff}\quad
  (M_{\mathrm{AW}},\hist)\models\tau(\varphi),
\]
where $M_{\mathrm{AW}}$ is the source gossip model on the same agent set and root state.
\end{lemma}

\begin{proof}[Proof sketch]
Lemma~\ref{lem:lottery-transparency} first replaces the ELGM by its lottery-free Kripke structure.  Definition~\ref{def:aw-source-instance} identifies that structure's valuation and accessibility relations with those of $M_{\mathrm{AW}}$.  A structural induction on $\varphi$ then gives the result: the atomic case is the defining clause for $V(S_{ab})$, Boolean cases are immediate, and the knowledge case uses the common relation $\sim^{\mathrm{AW}}_a$ and the translated $K_a$ clause. 
\end{proof}

\begin{theorem}[Guard decidability for the source-compatible fragment]
\label{thm:guard-decidability}
Let $\ELGM$ be an Apt--Wojtczak source instance, let $\hist\in H_{\mathrm{src}}$, let $\alpha\in C_{\mathrm{src}}$, and let $\mathsf{pre}(\alpha)\in\vocabKsrc$.  Given finite encodings of $\alpha$, $\hist$, and $\mathsf{pre}(\alpha)$, checking $\Guard_{\ELGM}(\alpha,\hist)$ is decidable.
\end{theorem}

\begin{proof}[Proof sketch]
Because $H_{\mathrm{src}}$ contains every finite source call sequence, $\hist\cdot\alpha\in H_{\mathrm{src}}$, so the structural part of the Guard is decidable.  By Lemma~\ref{lem:source-correspondence}, the remaining precondition test is equivalent to satisfaction of $\tau(\mathsf{pre}(\alpha))\in\Lwn$ in $M_{\mathrm{AW}}$.  Apt and Wojtczak prove that this satisfaction problem is decidable for every finite source call sequence and every formula in $\Lwn$~\citep{apt2017epistemic}.  No complexity-class bound is claimed here.
\end{proof}

\begin{remark}[What the theorem does not say]
\label{rem:complexity-scope}
The theorem does not cover probability threshold guards, arbitrary factual atoms outside $\vocabKsrc$, nested individual knowledge, common knowledge, arbitrary admissible ELGMs with an ineffective or protocol-restricted history domain, phone-number exchange, or arbitrary LLM tool traces.  It also does not import complexity results for dynamic EPLS model checking~\citep{vaneijck2014epls}; that is a different problem.
\end{remark}

\section{State Updates and Candidate-Action Selection}
\label{sec:updates-llm}

There are two different dynamics in the model.  A call changes the factual gossip state and hence the histories and views.  A lottery reweighting changes an agent's graded uncertainty over histories without changing the underlying history domain, view relation, or valuation.

\subsection{Structural Call Update}

For a call $ab\in\Acts_H(\hist)$ permitted by $\Guard_{\ELGM}$, the structural update appends the call to the actual history and updates the factual gossip state by the usual push-pull rule.  This paper uses the direct call-history update $\hist\mapsto\hist\cdot ab$, which is defined because $ab\in\Acts_H(\hist)$.  We do not claim that a singleton dynamic epistemic logic (DEL) event model is sufficient to encode the information exchanged in a gossip call; a correct action model would need to encode the participants' observations.

\subsection{Probabilistic Reweighting}

The probabilistic update is a normalized product rule, inspired by dynamic update with probabilities~\citep{vanbenthem2009probability}.  It is a modeling definition for ELGMs, not a theorem imported wholesale from that literature.  The factors below are positive likelihood factors.  This notation does not assert that they are normalized probability kernels over a separately defined event or observation space.

\begin{definition}[Three-source lottery reweighting]
\label{def:three-source-update}
Let $e$ be an event and $o_a$ the observation received by agent $a$.  Suppose the prior lottery is admissible and the factors
\[
  P_{\mathrm{occ}}(e\mid \hist)>0,
  \qquad
  P_{\mathrm{obs}}(o_a\mid e,\hist)>0
\]
are bounded and chosen so that the normalizer below is finite and nonzero.  Define
\[
  \lot{a}'(\hist)=\frac{1}{Z_a}\,\lot{a}(\hist)\,
  P_{\mathrm{occ}}(e\mid \hist)\,
  P_{\mathrm{obs}}(o_a\mid e,\hist),
\]
where
\[
  Z_a=\sum_{\hist'\in H}\lot{a}(\hist')
  P_{\mathrm{occ}}(e\mid \hist')P_{\mathrm{obs}}(o_a\mid e,\hist').
\]
The stated finite, nonzero normalizer is an additional assumption: classwise admissibility of the prior alone does not imply it.  Under the positivity and boundedness conditions above, the resulting lottery has finite positive class sums and is therefore admissible.
\end{definition}

\begin{lemma}[Positive reweightings preserve guards]
\label{lem:positive-kernel}
Suppose every lottery that is reweighted has the form $\lot{a}'(\hist)=\lot{a}(\hist)k_a(\hist)$, where $0<k_a(\hist)<\infty$, and every resulting lottery is admissible.  If $H$, $V$, and all accessibility relations $\Racc{a}$ are unchanged, then every $\varphi\in\vocab_K$ has the same truth value before and after the reweighting.  In particular, every Guard in Definition~\ref{def:guard} has the same truth value.
\end{lemma}

\begin{proof}[Proof sketch]
The update changes numeric weights but preserves positive support and leaves $H,V,$ and all $\Racc{a}$ fixed.  Lemma~\ref{lem:lottery-transparency} says $\vocab_K$ truth depends only on those shared structures under admissibility.  Hence Guard truth is invariant.  
\end{proof}

\subsection{LLM Candidate-Action Policies}

For the formal core, an LLM policy is represented as a stochastic generator of gossip-call candidates
\[
  \LLMpol(\alpha\mid \hist), \qquad \alpha\in C.
\]
At a decision step, the LLM submits $\alpha$; the Guard permits or denies it.  Permitted calls update the call history.  A denied candidate does not become safe merely because the LLM assigned it high probability.  This is the intended separation between graded model confidence and symbolic knowledge.  Natural-language messages, retrieval calls, and other action types lie outside this formal transition system until their parsers and operational interfaces are specified.

Representation-editing methods such as ARE~\citep{zhang2024are} can change an LLM's distribution over candidate actions.  This paper does not identify ARE with the lottery reweighting in Definition~\ref{def:three-source-update}, and no theorem below depends on ARE.  Connecting a neural intervention to the positive-kernel lemma requires an empirical mapping from model behavior to epistemic alternatives, together with a check of the lemma's assumptions.

\section{Decidability and Conditional Progress}
\label{sec:decidability-termination}

Guard checking is decidable only in the source-compatible fragment.  Termination needs further protocol assumptions.  We state a general condition for later transfer instead of applying a gossip bound to all LLM-agent executions.

\begin{theorem}[Conditional ranking termination]
\label{thm:conditional-termination}
Let $(X_t)_{t\geq 0}$ be a discrete-time execution process on a state space $X$, adapted to a filtration $(\mathcal{F}_t)_{t\geq 0}$ that records the complete execution prefix through step $t$, with $X_0=x_0$.  Suppose there is a ranking function $\rho:X\to\{0,1,\ldots,B\}$ with goal set $G=\rho^{-1}(0)$, and define the hitting time $T:=\inf\{t\geq 0:X_t\in G\}$, with $T=\infty$ if this set is empty.  Let $\beta:X\setminus G\to C$ select a Guard-permitted call at each non-goal state, such that every outcome of executing $\beta(x)$ reaches $x'$ with $\rho(x')<\rho(x)$.  Conditional on $\mathcal{F}_t$, the policy must select $\beta(X_t)$ at the next candidate step with probability at least $\varepsilon>0$ whenever $X_t\notin G$.  Finally, every other permitted call and every rejected candidate must leave the rank nonincreasing.  Then $\mathbb{E}[T]\leq\rho(x_0)/\varepsilon\leq B/\varepsilon$.
\end{theorem}

\begin{proof}[Proof sketch]
At any non-goal state, the waiting time to the next strict rank decrease is stochastically dominated by a geometric random variable of mean $1/\varepsilon$.  The remaining assumptions prevent a different candidate from increasing the rank or undoing progress.  At most $\rho(x_0)$ strict rank decreases are needed before rank zero is reached.  Linearity of expectation gives the bound.  
\end{proof}

\begin{remark}[Instantiating the rank]
Classical gossip results can instantiate the theorem only after the implemented transition system exhibits the required rank, nonincreasing alternative transitions, and lower policy bound.  Complete-graph push-pull Learn New Secrets (LNS) protocols have source-specific call bounds~\citep{ditmarsch2017gossip}; dynamic partial networks have different success conditions~\citep{ditmarsch2019dynamic}.  This paper therefore does not state a universal closed-form bound for arbitrary guarded LLM-agent executions.
\end{remark}

\section{Related Work and Future Work}
\label{sec:related-limitations}

\paragraph{Epistemic gossip and DEL.}
Dynamic epistemic logic provides the standard semantics for information-changing events~\citep{vanditmarsch2007del}.  Epistemic gossip protocols specialize this tradition to pairwise calls and higher-order knowledge of secrets~\citep{ditmarsch2017gossip,herzig2017gossip,cooper2019epistemic}.  EPLA model imports the call-history perspective, but adds lottery weights so that agents can have graded uncertainty over histories.

\paragraph{Epistemic probability and probabilistic update.}
EPLS motivates the use of lotteries and the identification of knowledge with probability one~\citep{vaneijck2014epls}.  Earlier distributed-systems work shows that the choice of agent probability spaces matters in runs-and-systems models~\citep{halpern1993knowledge}; ELGMs instead fix histories as worlds and source-style views as the epistemic relation for a narrow gossip setting.  Kooi~\citep{kooi2003probabilistic} provides broad probabilistic dynamic epistemic logic background.  Van Benthem, Gerbrandy, and Kooi~\citep{vanbenthem2009probability} specifically distinguish prior, occurrence, and observation probabilities in dynamic update.  EPLA model uses these ideas in a restricted gossip setting and keeps the full-support assumptions explicit.

\paragraph{LLM agents and formal methods.}
Recent work contrasts formal multi-agent systems (MAS) theory with LLM-agent practice~\citep{lamalfa2025}.  Zhang et al.~\citep{zhang2024verify} provide a formal-methods roadmap for trustworthy AI agents.  Yu et al.~\citep{yu2025formal} instead study model checking for multi-agent systems modeled in an epistemic process calculus.  Contract-like or neurosymbolic layers for agents make a similar engineering move: a learned model generates candidate behavior, while a symbolic layer constrains execution~\citep{leoveanu2025dbc}.  EPLA model instantiates this pattern with an epistemic state layer, a guarded action interface, and learning interfaces.

\paragraph{Representation editing.}
ARE is an empirical adversarial representation engineering method for editing LLM behavior~\citep{zhang2024are}.  It uses hidden state representations, a discriminator, and fine-tuning objectives.  In EPLA model, it is an architectural component whose effects must be tested empirically after implementation.


Implementation and empirical assessment remain the main line of our future work, which is currently being conducted.

\section{Conclusion}
\label{sec:conclusion}

EPLA model combines a Policy LLM that generates candidate actions, a Symbolic Guard that controls execution, and an epistemic state layer that makes the Guard's restricted conditions precise.  The central technical result is lottery transparency for the knowledge fragment.  Under admissibility and full support, positive lottery weights do not affect which knowledge-fragment Guards are true.  For the explicitly defined modal-depth-one source fragment, the correspondence with Apt--Wojtczak's model makes Guard checking decidable.  Positive admissible reweighting therefore preserves these Guards, and the conditional ranking result bounds expected progress when its explicit policy and rank assumptions hold.  The broader architecture integrates conditional belief, retrieval, diagnostic Guard feedback, representation editing, and LTL-constrained policy learning. To support the effectiveness of our approach, we are currently conducting limited experimental validation.

\bibliographystyle{plainnat}
\bibliography{references}

\end{document}